\documentclass[11pt]{article}
	
	\newcommand{\blind}{0}
	
    \makeatletter
    \renewcommand\section{\@startsection {section}{1}{\z@}%
                                       {-3.5ex \@plus -1ex \@minus -.2ex}%
                                       {2.3ex \@plus.2ex}%
                                       {\normalfont\fontfamily{cmr}\fontsize{16}{19}\bfseries}}
    \renewcommand\subsection{\@startsection{subsection}{2}{\z@}%
                                         {-3.25ex\@plus -1ex \@minus -.2ex}%
                                         {1.5ex \@plus .2ex}%
                                         {\normalfont\fontfamily{cmr}\fontsize{14}{17}\bfseries}}
    \renewcommand\subsubsection{\@startsection{subsubsection}{3}{\z@}%
                                        {-3.25ex\@plus -1ex \@minus -.2ex}%
                                         {1.5ex \@plus .2ex}%
                                         {\normalfont\normalsize\fontfamily{cmr}\fontsize{14}{17}\selectfont}}
    \makeatother
	
	\usepackage{amsmath}
	\usepackage{graphicx}
	\usepackage{enumerate}
	\usepackage{natbib} 
	\usepackage{url} 
        \usepackage{xcolor}
    \usepackage{hyperref} 
	
		\usepackage{amsfonts, amsthm, latexsym, amssymb, nicefrac,algpseudocode,algorithm2e, multirow, bm, todonotes}

\usepackage{tikz}
\usetikzlibrary{decorations.pathreplacing,calc}
\newcommand{\tikzmark}[1]{\tikz[overlay,remember picture] \node (#1) {};}
\newcommand*{\AddNote}[4]{%
    \begin{tikzpicture}[overlay, remember picture]
        \draw [decoration={brace,amplitude=0.5em},decorate,ultra thick,red]
            ($(#3)!(#1.north)!($(#3)-(0,1)$)$) --  
            ($(#3)!(#2.south)!($(#3)-(0,1)$)$)
                node [align=center, text width=0.6cm, pos=0.5, anchor=west] {#4};
    \end{tikzpicture}
}%

\DeclareMathOperator*{\argmax}{arg\,max}

\newcommand \EE {\mathbb{E}}

\newcommand \RR {\mathbb{R}}

\newcommand{\norm}[1]{\left\lVert#1\right\rVert}

\newcommand{\bmn}[1] {\bm{\mathrm{#1}}}

\newcommand{\T}{^\top}

\newcommand{\by}{\times}

\newtheorem{theorem}{Theorem}
\newtheorem{lemma}{Lemma}

\begin{document}
		
		\def\spacingset#1{\renewcommand{\baselinestretch}%
			{#1}\small\normalsize} \spacingset{1}
		
		\if0\blind
		{
			\title{\bf Adversarial Client Detection via Non-parametric Subspace Monitoring in the Internet of Federated Things}
			\author{Xianjian Xie $^a$, Xiaochen Xian $^b$, Dan Li $^c$ and Andi Wang $^d$ \\
			$^a$ School of Computing and Augmented Intelligence, Arizona State University, Tempe, AZ \\
             $^b$ School of Industrial and Systems Engineering, University of Florida, Gainsville, FL \\ 
             $^c$ Department of Industrial Engineering, Clemson University, Clemson, SC \\ 
             $^c$ School of Manufacturing Systems and Networks, Arizona State University, Tempe, AZ}
           \date{}
        			\maketitle   
		} \fi
		
		\if1\blind
		{

            \title{\bf Adversarial Client Detection via Non-parametric Subspace Monitoring in the Internet of Federated Things}
            
			\author{Author information is purposely removed for double-blind review}
\bigskip
			\bigskip
			\bigskip
			\begin{center}
				{\LARGE\bf Adversarial Client Detection via Non-parametric Subspace Monitoring in the Internet of Federated Things}
			\end{center}
			\medskip
		} \fi
		\bigskip

\begin{abstract}
The Internet of Federated Things (IoFT) represents a network of interconnected systems with federated learning as the backbone, facilitating collaborative knowledge acquisition while ensuring data privacy for individual systems. The wide adoption of IoFT, however, is hindered by security concerns, particularly the susceptibility of federated learning networks to adversarial attacks. In this paper, we propose an effective non-parametric approach FedRR, which leverages the low-rank features of the transmitted parameter updates generated by federated learning to address the adversarial attack problem. Besides, our proposed method is capable of accurately detecting adversarial clients and controlling the false alarm rate under the scenario with no attack occurring. Experiments based on digit recognition using the MNIST datasets validated the advantages of our approach. 


\end{abstract}
	\noindent%
	{\it Keywords:} Federated Learning; Data-stream monitoring; Non-parametric method. 

	\spacingset{1.5} 

\section{Introduction}

Recently, interconnected systems with the capability of collecting data from sensors during their operation and exchanging them over the Internet become increasingly common. 
Internet of Federated Things (IoFT, \citet{kontar2021internet}), regarded as the future of the Internet of things, refers to the inter-connected systems where the individual systems have computational capabilities and collectively gain knowledge and optimize their operations. Currently, the IoFT sees its application in many areas, from distributed manufacturing systems to autonomous transportation systems to electric power systems \citep{rahman2023icn, alnajar2023tactile, hegiste2022application}.

The backbone of the IoFT is Federated Learning (FL),  a category of machine learning methods where the individual systems cooperatively develop a collective machine learning model using the data collected locally by individual systems, without sharing these data directly with each other. The coordination of this process is facilitated by a central system, which assumes the primary role of collecting and aggregating the transmitted updates from individual systems. Transmitted updates capture the modifications made to the model's parameters based on locally collected data of each individual system. The benefit of FL is that all participating systems have access to the model developed using data collected from all systems while keeping the data of each individual system confidential. 

One barrier to the wide adoption of IoFT is the security concerns \citep{azad2021preventive, sharma2019survey}.
Among multiple security issues, this article addresses the following specific one - the FL networks are susceptible to adversarial attacks from the individual systems, where the adversarial attacks refer to malicious behaviors to compromise the fidelity of the model \citep{bhagoji2019analyzing}. 
Given the decentralized nature of FL, each system participating in the FL process may pose as a potential adversary. Malicious behaviors may manifest in various forms, such as modifying local data, applying inappropriate hyper-parameters during local computing, and transmitting wrong information to the central system \citep{mothukuri2021survey}. 


As reviewed in the next section, there are two major approaches to defend against adversarial attacks: alleviating the effect of adversarial participants through robust learning and identifying the adversarial participants through anomaly detection. This article considers the second approach, which  is able to eradicate the root cause of adversarial attacks. Essentially, the identification of the adversarial attacks in IoFT naturally connects with the problem of data stream monitoring. Determining whether a system within the IoFT is under attack relies on the transmitted updates from individual systems, necessitating concurrent monitoring of these updates. As a data stream monitoring problem, the identification of adversarial attacks is to detect the changes in the transmitted updates quickly and automatically after the attack occurs. Nevertheless, this data stream monitoring problem has several unique characteristics and challenges. 

First, when the underlying models of the FL process are deep learning models, the information transmitted among the systems is typically high dimensional due to the presence of tens of thousands of model parameters \citep{rodriguez2023survey}. 
Although monitoring algorithms are developed for observations of similar sizes such as images and multiple profiles \citep{liu2015adaptive, wang2018spatial}, all of them leverage the specific structures of the observations for effective monitoring, as will be reviewed in the next section. How to utilize the structure of the transmitted updates of the IoFT is still an open problem.  

Second,  an adversarial detection method that is applicable practically would require a limit on the false alarm rate. This false alarm rate is described by the average run length (ARL) of the chart when no adversarial attack occurs. In statistical process monitoring, the in-control ARL can be specified with a proper value of the control limit, based on the proper assumptions on the probabilistic model of the data streams, such as Poisson \citep{wang2017statistical} or Normal \citep{mei2010efficient}. However, this requirement is very hard to be satisfied for adversarial attack identification in IoFT, where the systems collectively solve an optimization problem. Essentially, the transmitted updates among systems are the intermediate values of the iterative algorithm that solves the global optimization problem, instead of the data that follows specific models. 
There is no convincing probabilistic model that describes the underlying distribution of the transmitted updates. Therefore, it is critical yet challenging to develop an adversarial client detection algorithm with a specific in-control ARL when a prescribed control limit is given. 

In this paper, we will address the aforementioned challenges in detecting adversarial participating systems within IoFT with the following two strategies and develop a monitoring system, Federated adversarial client detection with the Ranks of Residuals (FedRR). 


\begin{enumerate}
    \item  We will utilize the following observation from the FL literature \citep{azam2021recycling, konevcny2016federated} for developing monitoring schemes: for deep learning models with highly redundant model parameters, the linear subspace spanned by the model parameter updates (the sum of gradients of the loss function over multiple epochs) generated across the stochastic gradient descent (SGD) algorithm during the FL process is typically low dimensional. 
    Note that our study validates a new usage of the low-dimensional structure of the transmitted updates in FL, 
    different from the existing usage of reducing the communication overhead in the FL process. 

\item To address the lack of an appropriate probabilistic model for the distribution of transmitted updates and develop charts with a limited false alarm rate, our method leverages an effective  non-parametric statistics approach: transforming the transmitted updates after the dimensional reduction into rank statistics, such that the comparison of the participating systems can be described statistically for evaluating the adversarial attack behavior. With this idea, we develop a non-parametric data stream monitoring approach for the ranks of the participating systems.  

\end{enumerate}

In this paper, we focus on a centralized FL setting, where a central computing system (server) communicates with the participating systems (clients) and coordinates them in sequential steps of the learning algorithms. In this setup, the central server undertakes the task of adversarial attack detection using the information retrieved from individual clients. We implement our methods within the most commonly used FL framework FedAvg \citep{mcmahan2017communication}, showing that our methods are easily applicable to existing FL frameworks. We will focus on untargeted attacks that aim to impair the overall learning model performance \citep{fang2020local}, as opposed to altering the model's behavior exclusively for specifically targeted machine-learning subtasks while maintaining good accuracy for the primary task \citep{bouacida2021vulnerabilities}. The untargeted attacks include both data poisoning attacks, which involve the modification of local data for one or more clients, and model poisoning attacks \citep{rodriguez2023survey}, which directly manipulate the model parameter updates transmitted by clients to the server, as illustrated in Figure~\ref{attack}. 
However, we shall see that the overall strategies from the FedRR approach can be extended to the general adversarial client detection problem in IoFT with other FL backbones.


\begin{figure}[!htbp]
  \centering
  \includegraphics[scale=0.28]{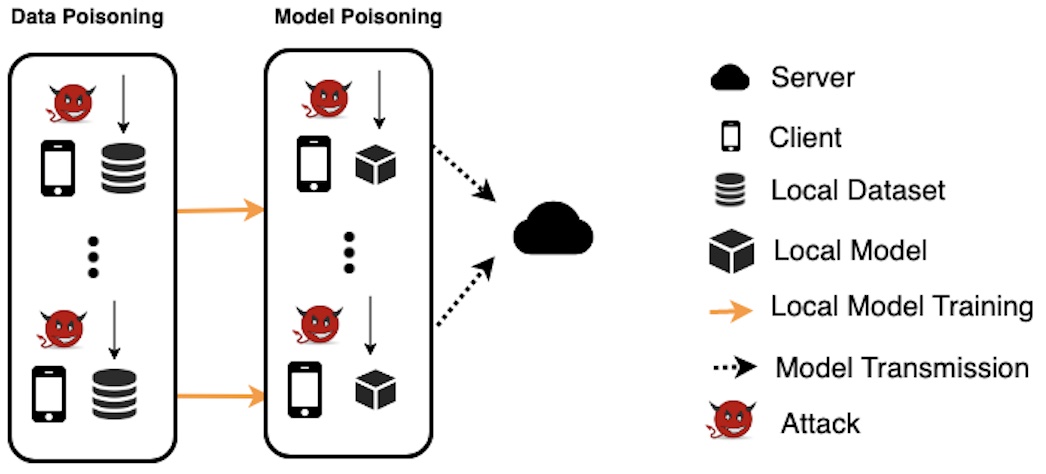}
  \caption{Type of Attacks}
  \label{attack}  
\end{figure}



We summarize our contributions as follows:
\begin{itemize}
  \item Our study introduces an effective non-parametric method that leverages the low-rank features of the space of transmitted updates generated by FL to address the adversarial client detection problem.
  \item Unlike existing client selection methods that rely on classification, our proposed monitoring technique is capable of accurately detecting adversarial clients with precise ARL under no-attack scenarios 
  \item To substantiate the effectiveness of our approach in adversarial client detection, our method is validated with the experiments using the MNIST dataset.
  
\end{itemize}

The remaining part of the article is organized as follows. In Section 2, we briefly introduce the background and review literature related to the FedAvg algorithm, methods of defending adversarial attacks, and big data monitoring. Section 3 introduces our problem setup and the methods in detail, and Section 4 provides experimental studies on digit recognition. Finally, Section 5 concludes the article.

\section{Background and Literature Review}

The background and literature review includes three components: (1) the FL framework and the FedAvg algorithm, a widely-adopted algorithm for the FL process, based on which we develop our FedRR algorithm; (2) the existing approaches for defending adversarial attacks in IoFT, using robust aggregation function and anomaly detection approach; (3) the existing high-dimensional data monitoring approaches. 

\subsection{General setup of the centralized federated learning and the FedAvg algorithm}\label{sec:fedavg}


A centralized IoFT typically comprises two primary components: a group of clients  that maintain and process confidential datasets  and a central server that aggregates and coordinates the information from the clients. The workflow of the centralized FL essentially involves the following four stages: (1) the server transmits the current model parameters to each client in the IoFT, (2) each client updates the parameters individually with its local dataset, (3) each client transmits its updated parameters back to the central server, and (4) the server aggregates the updated parameters of all clients and computes the new model parameters \citep{rodriguez2022dynamic}. Figure~\ref{workflow} illustrates the general FL workflow, and a detailed review of the recent advances of FL can be found in \citet{kairouz2021advances}. 




\begin{figure}[!htbp]
  \centering
  \includegraphics[scale=0.15]{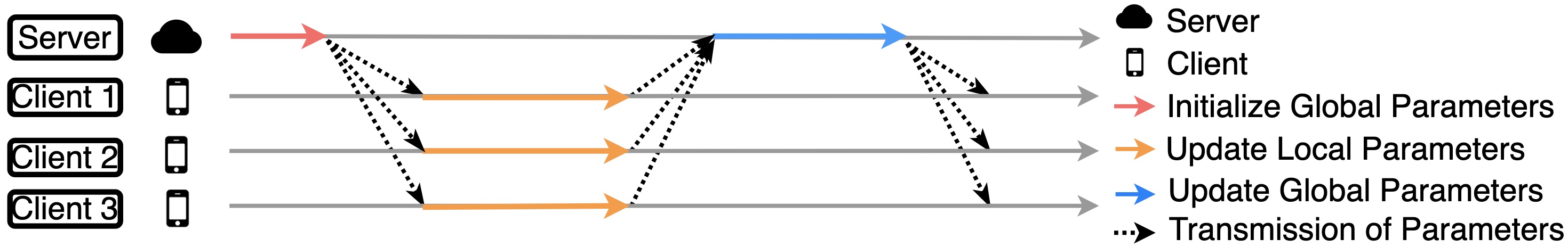}
  \caption{Federated Learning Workflow}
  \label{workflow}  
\end{figure}

A specific implementation of the above FL framework is FedAvg \citep{mcmahan2017communication}. Assume that there are $K$ clients participating in the learning process and that  each  client $k = 1,\ldots, K$ in communication round $t$ updates its local parameters 
using its own dataset $ \mathcal D_{t,k} = \{\bmn x_{t,k}^{(i)}, y_{t,k}^{(i)}\}_{i=1}^n$. They collaboratively solve the following machine learning problem \ref{eq:fedavg}:

\begin{equation}
\label{eq:fedavg}
\text{minimize} \quad F(\bmn w) =  \frac{1}{K} \sum_{k=1}^{K}\EE_{(\bmn x,y)\sim P_k}[ l(\bmn x,y;\bmn w)]. 
\end{equation}


\noindent In this representation, $\bmn w$ is the vector of the model parameters, and  $l(\bmn x,y;\bmn w)$ is the loss function defined by the specific machine learning model on a data sample $(\bmn x,y)$, and $P_k$ is the population of the data obtained from client $k$. In the $t$-th round, the server transmits the current model parameters $\bmn w_{t-1}$ (either the aggregation of the last round when $t>1$, or the initialized model parameters when $t=1$) to $K$ clients. Each client $k$ updates the received model parameters to $\bmn w_{t}^{(k)}$ using the SGD scheme. Finally, the server aggregates all clients’ updates and generates the updated global model parameters by $\bmn w_{t} = \frac 1 {K}\sum_{k=1}^{K} \bmn w_{t}^{(k)}$. The FedAvg is outlined in Algorithm \ref{FedAvg} with a total of $T$ rounds of communication. In this article, we will base our FedRR procedure on FedAvg algorithm. 

\RestyleAlgo{ruled}
\SetKwProg{Init}{Initialization:}{}{}
\SetKwProg{Loc}{Local Update:}{}{}
\SetKwProg{Agg}{Global Aggregation: }{}{}



	\spacingset{1.0} 
\LinesNumbered
\begin{algorithm}
\label{FedAvg}
\DontPrintSemicolon
\KwIn{The local dataset $\mathcal D_{t,k} = \{\bmn x_{t,k}^{(i)},y_{t,k}^{(i)}\}_{i=1}^{n}$ collected by client $k$ of the communication round $t$. Initial value of the model parameters $\bmn w_0$}
\For{$t = 1, \ldots, T$}{
The server transmits the current $\bmn w_{t-1}$ to all clients $k = 1,..., K$\;
\Loc{}{
\For{$k = 1,\ldots,K$}{
      Server $k$ updates the model parameters based on $\bmn w_{t-1}$ and data $\mathcal D_{t,k}$ using the SGD procedure and obtains $\bmn w_{t}^{(k)}$\;
      The client $k$ transmits the updated parameters $\bmn w_{t}^{(k)}$ to the server\; 
    }
}
\Agg{}{
The server aggregates the updated parameters from all clients $\bmn w_{t} = \frac 1 K \sum_{k=1}^{K} \bmn w_{t}^{(k)}$\;
}
}
The output of the FedAvg procedure is $\bmn w = \bmn w_T$\;
\caption{FedAvg Algorithm \citep{mcmahan2017communication}}
\end{algorithm}

	\spacingset{1.5} 

\subsection{Defenses against Adversarial Attacks}

In cyber security research, there are a plethora of defensive methods relying on data inspection technology. However, they cannot be implemented in FL due to the absence of central server access to local client data \citep{rodriguez2022dynamic}. Two commonly used defenses against adversarial attacks for FL systems are robust aggregators and anomaly detection.

\subsubsection{Adversarial attack defense with Robust Aggregators }



A robust aggregator refers to the use of an aggregation function that differs from the arithmetic mean used in FedAvg. The intuitive idea is to decrease the weights of clients during the aggregation step if they generate extreme updates, to protect the fidelity of the model parameters. \citet{yin2018byzantine} suggests using median and trimmed-mean as the aggregation function, as they are less sensitive to sporadic faulty updates from adversarial clients. 
Similar approaches include \citet{wu2020federated} and \citet{chen2017distributed} that use geometric median, \citet{blanchard2017machine} that discards the clients generating outlying updates, and \citet{sun2019can} that clips model parameter updates in the global aggregation if the norm exceeds a threshold. \citet{guerraoui2018hidden} introduces a detailed multi-step client selection criterion applicable to multiple aggregation rules, that guarantees rapid convergence to the true solution under certain assumptions of the adversarial situations.





Some methods enhance the robust aggregator by  accounting for the updates from the historical rounds. For example, the Adaptive Federated Averaging (AFA) method \citep{munoz2019byzantine} uses Hidden Markov Models (HMM) to model the updates of different clients, assess their performance, and assign appropriate weights to the updates based on their respective performance
during aggregation. However, the HMM model may not adequately capture how the updates of the parameters from each client vary during the FL process, given the diverse form of the objective function as well as the arbitrary selection of the initial parameter values. 



Lastly, \cite{park2021sageflow} suggests Sageflow, which weighs each client's model updates according to how well it performs on a public dataset 
and assigns weights to each model update accordingly. However, the requirement for a public dataset to be available may not be practical in certain scenarios. 

\subsubsection{Adversarial attack defense with anomaly detection approach}

A common problem with robust aggregators is that they cannot completely remove the impact of adversarial attacks. 
Anomaly detection strategies aim to accurately identify adversarial clients and exclude them from model aggregation. The key idea is to identify whether the client is adversarial based on its transmitted parameters. 
\citet{andreina2021baffle} proposes that each client individually evaluates whether the updates improve the model performance using its private data and votes to determine whether the current update is acceptable.
\citet{cao2019understanding} constructs a graph by calculating the pairwise Euclidean distance between model updates and then uses the Bron–Kerbosch algorithm \citep{bron1973algorithm} to identify the maximum clique. 
\citet{zhao2020pdgan} proposes to reconstruct the datasets from the model on the server using the model updates from the client, and then identify the adversarial clients by verifying the model accuracy using these datasets.
Common issues of the above approaches include significant computational and communication overhead, and that these anomaly detection approaches are based on flexible decision rules that forbid accurate evaluation of detection error rates.  




\subsection {High dimensional data monitoring approaches} \label{sec:hdmonitoring}
As we mentioned in the introduction, there is a close connection between the identification of adversarial clients and high-dimensional data monitoring, as transmitted updates between clients and the server are high-dimensional when first-order methods are used to optimize the deep learning models. A key aspect is to leverage the structure of both the data under the normal state and the anomalies. For example, \citet{paynabar2016change} studies the monitoring of multiple continuous functions with the mean shift; \citet{yan2018real} studies the monitoring of sparse anomalies on the smooth streaming profiles, and \citet{liu2015adaptive} and \citet{wang2018spatial} aim to find the sparse anomalies on independent data streams. Despite the different models and methods, these methods leverage the distinct structure of the data or the anomalies to design monitoring procedures and overcome the curse of dimension. While these assumptions in the above literature are valid in applications such as the monitoring of images or  multiple functional signals in manufacturing systems, the transmitted updates between the clients and the central server in the IoFT do not have these features. In order to develop an appropriate monitoring approach, the identification and utilization of the structure among these transmitted updates is the cornerstone.  

\section{Methodology}

Our goal of the research is to develop an effective method to identify adversarial clients among all $K$ clients, following the FedAvg setup introduced in Section~\ref{sec:fedavg}, so that we can exclude their updates in the aggregation step and avoid compromising the final performance of our model. 

We posit the fundamental assumption of independent and identically distributed (IID) data across all clients while acknowledging that adversarial clients form a minority within the participating systems. Moreover, our assumption narrows down the range of actions undertaken by adversarial clients to two distinct categories: data poisoning attacks and model poisoning attacks.

As we detailed in the introduction, the FedRR method is composed of two major considerations: (1) decreasing the dimension of the high-dimensional parameters through mapping them to a few summary statistics, and (2) constructing a non-parametric monitoring procedure that aggregates the updates from all clients. They will be introduced in Section~\ref{sec:low} and Section~\ref{sec:cusum}. In Section~\ref{sec:proc}, we present the overall monitoring structure and discuss the possible extensions of the method. 

\subsection {Leveraging the low dimensional structure of the data transmission} \label{sec:low}



Recall that our aim is to develop an adversarial client detection approach on whether transmitted updates
are generated by adversarial clients. In the general setup of IoFT, this task should be completed by the central server after they receive the updates from the clients before the aggregation. More specifically, when FedAvg algorithm is employed, the identification of the adversarial clients should be made after the central server receives the $\bmn w_t^{(k)}\in\RR^p$ from all clients $k=1,\ldots,K$ at time $t$. As introduced in Section~\ref{sec:hdmonitoring}, $\Delta \bmn w_t^{(k)} = \bmn w_t^{(k)} - \bmn w_{t-1} $ is high-dimensional and the effective monitoring and anomaly detection of $\Delta \bmn w_t^{(k)}$ hinge upon the low dimensional structure of it. 

To describe the low-dimensional structure of $\Delta \bmn w_t^{(k)}$, we use the results from \citet{azam2021recycling}, which states that  \emph{the gradients generated from the SGD process of an FL process typically stay in a low-dimensional subspace due to the over parametrization of the network}. This result has also been corroborated by other literature \citep{konevcny2016federated}. This result indicates that very few principal components of  $\text{span}\{\Delta \bmn w_t^{(k)}: k=1,\ldots, K; t=1,2,\ldots, T\}$ can account for a majority of the variability of the data, which motivates us to monitor the component 
of a $\Delta \bmn w_t^{(k)}$ perpendicular to this low-dimensional principal subspace, as an update \textit{not} calculated from the real data or \textit{not} calculated from the correct optimization procedure will be significantly outside this latent subspace. 

Based on  this idea, we divide the monitoring procedure into two phases: 
\begin{itemize}

\item \textbf{Phase I.} We assume that no adversarial attacks appear in the first $T_0$ rounds of communication. From the updates of these rounds, we therefore can estimate the updates' principal subspace. The central server collects the updates from all $K$ clients and obtain $\Delta\bmn W = [\Delta\bmn W_1\ \cdots \ \Delta\bmn W_{T_0}]\in \RR^{p \times T_0 K}$ where  $\Delta\bmn W_t = [\Delta\bmn w_t^{(1)}\ldots \Delta\bmn w_t^{(K)}]$. Then we  obtain $\bmn B\in \mathbb{R}^{p \times q}$, whose column vectors $ [\bmn b_1 \ldots \bmn b_q] $ are the first $q$ orthonormal principal components of all columns in $\Delta\bmn W$. 

\item \textbf{Phase II. } After the first $T_0$ round, we calculate a simple index that represents whether an update from a client is far away from the principal subspace. For each client $k$, we project $\Delta \bmn w_t^{(k)} = \bmn w_t^{(k)} - \bmn w_t $ onto the subspace defined by $\bmn B$. The length of the residual
\begin{equation}
\label{eq:3}
r_t^{(k)} = \norm{\Delta \bmn w_t^{(k)} - \bmn P_{\bmn B}\Delta \bmn w_t^{(k)}}_2
\end{equation}
characterizes whether the client $k$ is being attacked at round $t$, where $\bmn P_{\bmn B} = \bmn B\bmn B\T $ is the projection matrix onto the column space of $\bmn B$. 
Due to the high dimension of $\Delta \bmn w_t^{(k)}$, the expression $\bmn P_{\bmn B} \Delta \bmn w_t^{(k)}$ should be computed via the following equation, to avoid forming the matrix $\bmn P_{\bmn B}$ explicitly via $\bmn B\bmn B\T $ and thus avoid the expensive matrix-vector multiplications. This can decrease the computational burden from $O(p^2 q)$ to $O(pq)$. 
\[\bmn P_{\bmn B} \Delta \bmn w_t^{(k)} = \bmn B(\bmn B\T \Delta \bmn w_t^{(k)})\]
 
  \end{itemize}

In this scheme, the first $T_0$ rounds of communication is the start-up period of the monitoring scheme, and from the round $T_0 + 1$, the monitoring of the clients begins. 
  
\subsection {Data stream monitoring via non-parametric CUSUM procedure} \label{sec:cusum}

After the matrix $\bmn B$ is calculated after the first $T_0$ rounds, we are able to transform the updates after round $T_0$ into $r_t^{(k)}$. Our goal is to develop a monitoring scheme for identifying the clients under attack. To achieve this, we need an appropriate probabilistic model for $r_t^{(k)}$'s to assure the in-control adversarial detection
rate of the model. However, $r_t^{(k)}$'s are calculated based on intermediate updates in the optimization problem \ref{eq:fedavg}. The only knowledge on their distributions is that $r_t^{(1)}, \ldots, r_t^{(K)}$ are IID conditioning on the parameters at time $t$, $\bmn w_t$, given the IID data collected by different clients. This motivates us to transform the values of $r_t^{(k)}$ into the rank statistics
\[ \tilde{\bmn r}_t=[\tilde{r}_t^{(1)},\ldots,\tilde{r}_t^{(K)}], \]
where $\tilde{r}_t^{(k)} \in \{1,\ldots, K\}$ for all $k \in \{1,2,...,K\}$ are the rank of $r_t^{(k)}$ within $r_t^{(1)},\ldots, r_t^{(K)}$.
The following theorem describes the distribution of $\tilde {\bm r}_t, t=1,2,\ldots$.

    \begin{theorem}
      Assume that all samples from all clients are independent and that the population of each client satisfies $P_1 = P_2 =\cdots = P_K$ when there is no adversarial client. Then $\tilde{\bmn r}_t$ is a uniform random permutation, i.e., each of the $K!$ permutations is equally likely to appear. Moreover, $\tilde {\bmn r}_1, \tilde {\bmn r}_2, \ldots$ are all independent. 
      \label{thm}
    \end{theorem}


The proof of this theorem is given in the appendix. This theorem shows that transforming the updates to rank statistics leads to a closed-form probabilistic characterization of the clients' updates, which facilitates the development of a monitoring scheme with prescribed in-control performance. When a client is adversarial, its updates are more likely to become far away from the principal subspace and thereby lead to a higher rank among all $K$ clients. 

We propose a non-parametric data stream monitoring approach for the rank statistics $\tilde {\bmn r}_t$'s. The idea is to derive the likelihood for every client to be adversarial and then aggregate them \citep{xian2018nonparametric}. Given that $\tilde {r}_t^{(k)}$ is uniformly distributed on $\{1,\ldots, K\}$, we propose to transform it to a standard Normal random variable through $Z_t^{(k)} = \Phi^{-1}\left(\frac{\tilde {r}_{t}^{(k)} -U}{K}\right)$, where $\Phi(\cdot)$ is the cumulative distribution function of standard Normal distribution and $U$ is a random variable generated from $U(0,1)$. 
Then, we apply the CUSUM procedure on the sequences $\{Z_t^{(k)}, t=T_0+1,T_0+2,\ldots\}$ for any $k=1,\ldots, K$ and let $ S_{T_0}^{(k)} = 0 $, and $S_t^{(k)} = (S_{t-1}^{(k)} + Z_t^{(k)} - d / 2)_{+}$, where $\{S_{T_0+1}^{(k)}, S_{T_0+2}^{(k)}, \ldots\}$ represent the monitoring statistic of client $k$ and $d$ is the reference value of CUSUM procedure \citep{montgomery2020introduction}.
Note that we are only interested in detecting the upward shift of $Z_t^{(k)}$ with this formulation of $S_t^{(k)}$. 
Finally, we aggregate these local monitoring statistics by $S_t = \max\{S_t^{(k)}: k = 1,...,K\}$. If $S_t > H$, the alarm of adversarial attack is triggered, and we conclude that client $k^* = \argmax_{k}(S_t^{(k)})$ be the adversarial client, where $H$ is a threshold selected to yield an in-control expected ARL of $ARL_0$.  

\textbf{Determination of $H$ based on targeted in-control ARL. } An advantage of using rank statistics in the FedRR procedure is that the threshold $H$ can be obtained through an efficient offline simulation based on the target in-control $ARL_0$, the value of $d$, and the number of clients $K$. This simulation procedure only involves the sequential random permutations \textit{before} the FL process starts. Specifically, we can approximate the in-control ARL of a given threshold $H$ with the procedure given in Appendix B.

\subsection{Summary and discussions}\label{sec:proc}

The two stages of the FedRR procedure can be integrated into the FedAvg procedure as summarized in Algorithm~\ref{FedRR}, where the new steps are marked with brackets with stars (*). In this procedure, there are several tuning parameters that need to be determined: 

\begin{itemize}
    \item The number of start-up communication rounds $T_0$. For the optimal performance of the threat detection procedure, $T_0$ should be selected large enough to adequately describe the space of the updates so as to estimate the principal components $\bmn b_j$'s accurately. However, no adversarial attack should occur in the start-up period.  
    \item The dimension of the subspace $q$. The value of $q$ should be adjusted to account for a prescribed percentage, for example, 95\% or 99\%, of the variance from the $T_0 K$ updates generated from the start-up period. 
    \item The reference value of CUSUM procedure $d$. The value of $d$ should be adjusted based on the magnitude of the effect of the adversarial attack. It should be half the expected value of $Z_t^{(k)}$ after the attack happens. 
\end{itemize}

It is worth noting that regardless of the choice of $T_0$ and $q$, the FedRR approach can always achieve the targeted in-control ARL with the appropriate threshold $H$, as indicated by Theorem~\ref{thm}. However, an appropriate selection of $T_0$ and $q$ would increase the power of the adversarial client detection. 


As we presented in the introduction, the FedRR procedure relies on two key ideas: (1) dimensional reduction of the updates and (2) using rank statistics for comparing the updates from multiple clients. Based on these ideas, there are multiple potential extensions of the algorithm: 

\begin{itemize}
    \item Both ideas (1) and (2) can be extended and adapted to other centralized FL algorithms in processing the updates submitted by all clients. For decentralized FL systems, however, each client may apply the dimensional reduction on the updates submitted by other clients and calculate their ranks. Nevertheless, an aggregation of the decision from all decentralized clients is required and needs further studies. 
    \item Many IoFT systems in practice only have a few clients participating in each communication round. Under this setup, the ideas of FedRR are still applicable, while the ranks of those clients not participating in the updates should be regarded as missing data, and the monitoring approach needs to be modified. The ideas of data stream monitoring with partial observations can be borrowed \citep{liu2015adaptive}. 
    \item A major limitation of FedRR is the  assumptions of homogeneous clients. When the clients are heterogeneous, 
    the idea of the FedRR can be adapted if we can find appropriate summary statistics of the updates that satisfy the following conditions: (1) they decrease the dimension of the clients' updates, (2) they have the same distribution for all clients when they are not under attack, and (3) they have significantly different distributions when attacks occur. Intuitively, the summary statistics extract the common features from the samples collected from the heterogeneous clients. We will also leave this problem for future research. 

\end{itemize}





\spacingset{1.0} 
\RestyleAlgo{ruled}
\SetKwBlock{Server}{Network executes:}{end}
\SetKwProg{Init}{Initialization:}{}{}
\SetKwProg{Learn}{Learning Stage:}{}{}
\SetKwProg{Monitor}{Monitoring Stage:}{}{}

\begin{algorithm}
\label{FedRR}
\DontPrintSemicolon

    \Init{}{
        When $t=0$, the server initializes $\bmn w_0$ and $\Delta\bmn W\in \RR^{p \by 0}$. Prescribe $H$ using simulation. \;
    }
    \For (\tcc*[f]{Phase I: start up}) {$t = 1, \ldots, T_0$} {
The server transmits $\bmn w_{t-1}$ to all clients $k = 1,\ldots, K$.\;
\Loc{}{
\For{$k = 1,\ldots,K$}{
      The client $k$ updates the model parameters from $\bmn w_{t-1}$  using data $D_{t,k}$ with SGD procedure and obtains $\bmn w_{t}^{(k)}$.\;
      The client $k$ transmits the updated parameters $\bmn w_{t}^{(k)}$ to the server.\; 
    }
}
\Agg{}{
      {\color{red} The server collects the updates in matrix \tikzmark{top1} \; $\Delta\bmn W\gets [
                    \Delta\bmn W \  
                    \Delta\bmn  w_t^{(1)} \ \cdots\  \Delta\bmn w_t^{(K)} ]$,  where $\Delta \bmn w_t^{(k)} = \bmn w_t^{(k)} - \bmn w_{t-1}$ }. \tikzmark{bottom1} \tikzmark{right1}\;
The server aggregates the new parameters from clients $\bmn w_{t} = \frac 1 K \sum_{k=1}^{K} \bmn w_{t}^{(k)}$. \;
}
}
     {\color{red} The server applies PCA to the columns of $\Delta\bmn W\in \mathbb{R}^{p \by T_0 K }$ and obtains the first \tikzmark{top2} \tikzmark{right2} \;$q$   principal components, stored in the columns of $\bmn B = [\bmn b_1 \cdots \bmn b_q]\in \mathbb{R}^{p \by q}$. } \tikzmark{bottom2}\; 

        \For(\tcc*[f]{Phase II: monitoring}){$t = T_0 + 1, T_0 + 2, \ldots$}{
The server transmits the current $\bmn w_{t-1}$ to all clients $k = 1,..., K$.\;
\Loc{}{
\For {$k = 1,\ldots,K$}{
      Server $k$ updates the model parameters from $\bmn w_{t-1}$  using data $D_{t,k}$ with SGD procedure and obtains $\bmn w_{t}^{(k)}$.\;
      The client $k$ transmits the updated parameters $\bmn w_{t}^{(k)}$ to the server.\; 
}}
\Agg{}{
 {
 \color{red}
 
The server calculates $r_t^{(k)} = \norm{\Delta \bmn w_t^{(k)} - \bmn B(\bmn B\T \Delta \bmn w_t^{(k)}) }_2$ for all clients. 
          \tikzmark{top3} \tikzmark{right3}\; 
Transform $(r_t^{(1)}, \ldots, r_t^{(k)})$ to rank statistics $\left(\tilde{r}_t^{(1)},\ldots, \tilde r_t^{(k)} \right)$.\; 
 \For{$k=1,\ldots,K$}{
            Generate $U\sim U(0,1)$ and calculate $Z_t^{(k)} = \Phi^{-1}\left(\frac{\tilde {r}_t^{(k)}-U}{K}\right)$  and              $S_t^{(k)} = (S_{t-1}^{(k)} + Z_t^{(k)} - d/2)_{+}$.\; \label{FedRR:genZ}      
}
\If{$\max\{S_t^{(k)}: k = 1,...,K\} > H$}{
                Adversarial client $k$ is detected  at time  $t$. 
                \tikzmark{bottom3}      
            }          
}
    The server aggregates the updates from all clients $\bmn w_{t} = \frac 1 K \sum_{k=1}^{K} \bmn w_{t}^{(k)}$
}
}
\caption{FedRR Algorithm}
\AddNote{top1}{bottom1}{right1}{(*)}
\AddNote{top2}{bottom2}{right2}{(*)}
\AddNote{top3}{bottom3}{right3}{(*)}
\end{algorithm}

\spacingset{1.5}







\section{Experiments}

We conducted experiments to investigate and demonstrate the effectiveness of the FedRR approach. This experiment involves identifying numbers 0-9 from their images, through training samples collected from multiple clients. The experimental setups are introduced as follows. 

\noindent \textbf{Dataset.} We utilize the MNIST dataset \citep{lecun1998gradient} in this experiment. The MNIST dataset consists of a collection of $60,000$ training image samples and $10,000$ test samples. Each sample represents a grayscale handwritten digit ranging from 0 to 9, portrayed as a $28 \times 28$ pixel image with a single channel. Additionally, each sample is assigned a label corresponding to a digit between 0 and 9. 

\noindent \textbf{Model and the Federated Learning system.} Our experiments are performed with a CNN model with two convolution layers, two max-pooling layers, and two fully connected layers. This CNN model contains a total of $4,301,080$ parameters. The detailed model architecture is illustrated in Figure~\ref{cnn}. We study the case where there are $K=5$ clients in the FL system, the data owned by different clients are IID. The model training is performed using  the SGD algorithm with a fixed learning rate of $\eta = 0.001$. In each communication round, there are a total of $3$ epochs with  minibatches of size $128$.

\begin{figure}[!htbp]
  \centering
  \includegraphics[scale=0.16]{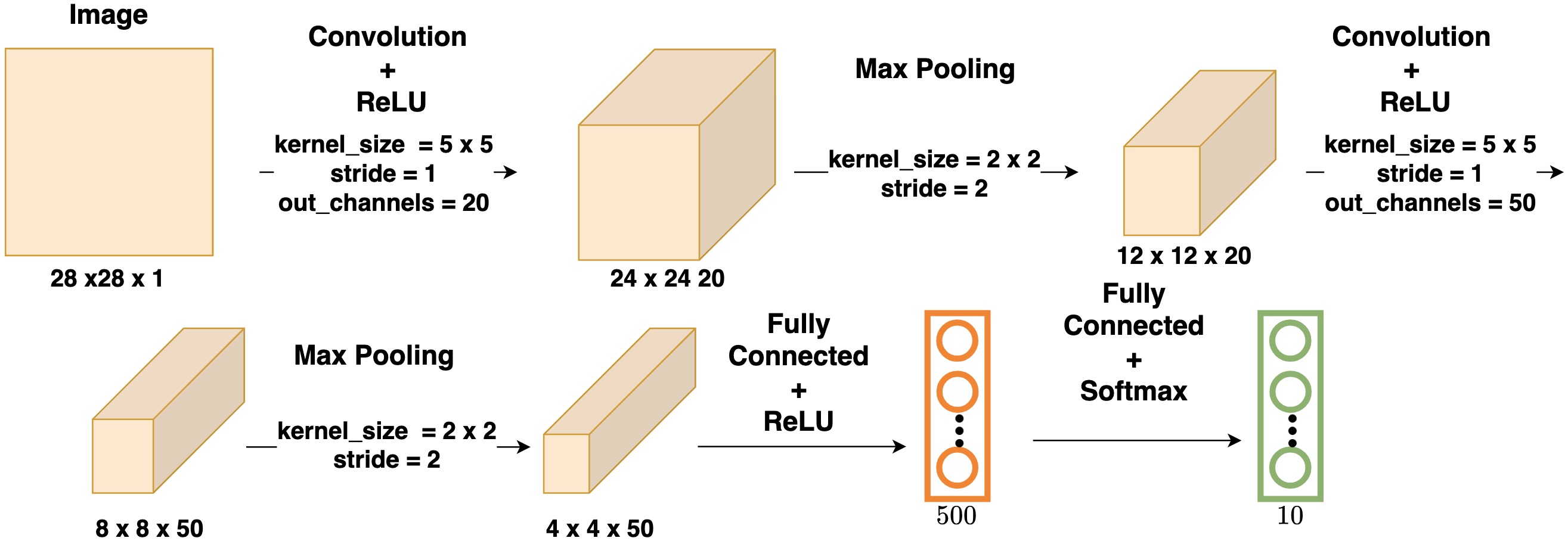}
  \caption{CNN Architecture}
  \label{cnn}  
\end{figure}


\noindent \textbf{Adversarial attack scenarios. } In this study, our aim is to evaluate the detection of adversarial attacks in the training process, instead of evaluating the model performance after training. We consider the following adversarial attack scenario: 5 clients collect both the images of the digits and the labels and collectively train a digit recognition model. 
In the first $T_0 = 50$ communication rounds, there are no adversarial attacks. During this time, the server estimates the updates' principal subspaces. From the communication round $T_0 +1$, we introduce adversarial attacks on one of the clients. We study three adversarial attacks in our experiments. 

\begin{enumerate}[(1)]
    \item \textit{Label flipping attack.} The labels of honest training examples of one class are randomly flipped to another class while the features of the data are kept unchanged. We control the proportion of data samples subject to label flipping attack on the adversarial client to $ratio=\frac{1}{250}$ 
    in our experiment. 
    \item \textit{Poisoning samples attack.} We add some noise to the local dataset's input image. We choose to add Gaussian noise $N(0.01,0.1)$ on every pixel of the image in our experiment.  
    \item \textit{model poisoning attack.} We add some noise to local model parameters before transmitting them to the server. The model poisoning attack in our experiment is deployed by adding independent Gaussian noise $N(0,0.0001)$ to every element of $\bmn w_t^{(k)}$. 
\end{enumerate}

\noindent \textbf{The benchmark methods. }
Recall that our proposed approach offers two benefits: a guaranteed false alarm rate when no attack is present due to the non-parametric approach, and the capability for effective detection through the projection on low-rank subspaces. Note that the first benefit among them is a basic requirement of all applicable monitoring schemes. Therefore, we will only compare our FedRR approach with a benchmark method not leveraging the knowledge  of low-rank updates.

To demonstrate the superiority of employing residuals to a low-dimensional subspace within the FedRR framework, we consider a benchmark method where the  updates' norms are monitored. Compared with the FedRR approach, Phase I of the algorithm is removed, and the residuals $r_t^{(k)}$ is substituted with $\norm{\Delta \bmn w_t^{(k)}}_2$, the magnitude of the change of parameters. 
We inherit the use of the non-parametric CUSUM procedure from FedRR to the benchmark method, so that the same error rate of the FedRR approach and the benchmark can be achieved without attacks, and thus the outcomes are comparable. 

\subsection{Validate the low-rank assumption of the updates }

In our preliminary study, we first validate the assertion from \citep{azam2021recycling} that the subspaces spanned by the transmitted updates generated by all clients during multiple communication rounds are low-dimensional. We save the transmitted updates of five clients generated during the first $500$ communication rounds, $\{\Delta \bmn w_t^{(k)}: k=1,\ldots, 5; t=1,2,\ldots, 500\}$. We then concatenate $\Delta \bmn w_t^{(k)}$ by column to form $\Delta \bmn W\in \mathbb{R}^{p \times (500\by 5)}$, where $p$ is the number of model parameters. We then perform principal component analysis on the first $TK$ columns of $\Delta \bmn W$, for $T=1,\ldots, 500$. Figure~\ref{low_rank} displays the number of components accounting for a specific percentage of the explained variance versus the total number of updates generated during FL training. Since $TK \ll p$, we know that the dimension of the subspace generated by the transmitted updates is at most $TK$. We discovered that the number of principal components required to explain the variance from 95\% to 99.9\% is significantly lower than $TK$, the total number of updates generated during model training. 
It indicates the potential of monitoring the residual 
for  detecting potential adversarial clients effectively.

\begin{figure}[!htbp]
  \centering
  \includegraphics[scale=0.58]{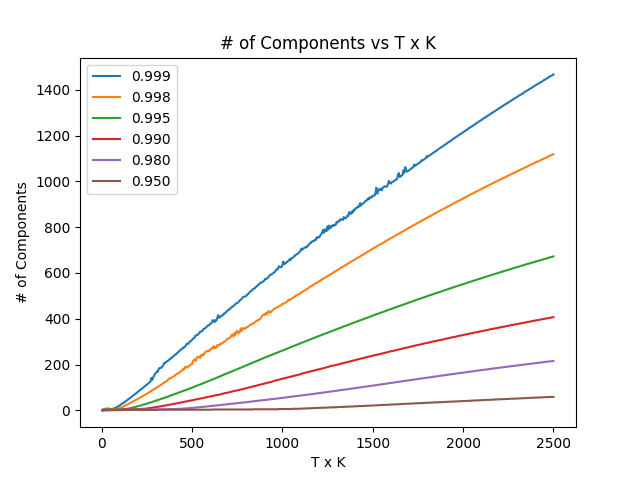}
  \caption{Low Rank Property}
  \label{low_rank} 
\end{figure}

\subsection {Compare the results between FedRR and the benchmark}
In our experiment, we compare the FedRR and the benchmark method with three values of $d$: 0.4, 0.5, and 0.6. For each case, we adjust the value of control limit $H$ using the approach detailed in Appendix B, and set the no-attack ARL of both the FedRR approach and the benchmark approach to be $ARL_0 = 30$. This value is smaller than most setups of the SPC literature due to the practical reason of long training time of the CNN model. 

The experimental results are shown in Table \ref{table:1}.  For all three adversarial attack modes, FedRR outperforms the benchmark method. Particularly in the case of the model poisoning attack mode, FedRR demonstrates significant superiority over the benchmark method. The model poisoning attack, which directly poisons the model parameters, induces a greater tendency for the transmitted updates of adversarial clients to deviate from the subspace spanned by the normal transmitted updates compared to the other two attack modes, which directly poisons the dataset and subsequently indirectly impact the model parameters. The benchmark method solely concentrates on the magnitude of the transmitted updates, failing to capture the information regarding deviations from the subspace spanned by the normal transmitted updates. Consequently, it exhibits lesser sensitivity compared to our method.



\begin{table}
\centering

\begin{tabular}{|c|l|l|l|l|l|}
\hline
\multicolumn{1}{|l|}{}                                                                          & $d$   & $ARL_0$ & $H$    & $ARL$ (FedRR) & $ARL$ (benchmark) \\ \hline
\multirow{3}{*}{\begin{tabular}[c]{@{}c@{}}Label Flipping Attack\\ ratio = 1/250\end{tabular}}  & 0.4 & 30   & 3.84 & 4.79(0.87) & 8.09(3.26)     \\ \cline{2-6} 
                                                                                                & 0.5 & 30   & 3.28 & 4.71(1.10) & 8.11(4.20)     \\ \cline{2-6} 
                                                                                                & 0.6 & 30   & 2.77 & 4.36(1.06) & 7.88(3.75)     \\ \hline
\multirow{3}{*}{\begin{tabular}[c]{@{}c@{}}Poisoning Samples Attack\\ $N(0.01,0.1)$\end{tabular}} & 0.4 & 30   & 3.84 & 5.05(1.43) & 8.44(2.71)     \\ \cline{2-6} 
                                                                                                & 0.5 & 30   & 3.28 & 4.72(1.87) & 8.12(2.71)     \\ \cline{2-6} 
                                                                                                & 0.6 & 30   & 2.77 & 4.63(1.86) & 8.30(3.06)     \\ \hline
\multirow{3}{*}{\begin{tabular}[c]{@{}c@{}}Model Poisoning Attack\\ $N(0,0.0001)$\end{tabular}}   & 0.4 & 30   & 3.84 & 4.37(0.82) & 21.10(3.39)    \\ \cline{2-6} 
                                                                                                & 0.5 & 30   & 3.28 & 4.33(0.86) & 22.25(3.69)    \\ \cline{2-6} 
                                                                                                & 0.6 & 30   & 2.77 & 4.14(1.21) & 21.75(3.63)    \\ \hline
\end{tabular}
\caption{The ARLs of the three setups when using FedRR and the benchmark method. The values in brackets are the standard deviation of the ARL. }
\label{table:1}
\end{table}

\section{Conclusion}
The security issue is a major concern for the wide adoption of IoFT. Adversarial client detection is an important and effective approach for eliminating attacks in FL systems, while the high dimension  and the lack of probability description of the transmitted updates are two major challenges. In this study, we propose a FedRR monitoring scheme that harnesses the low dimensional features of the space spanned by the transmitted updates. By reducing the dimension of transmitted updates and transforming them into rank statistics, FedRR enables effective data stream monitoring via a non-parametric CUSUM procedure. This framework offers distinct advantages over alternative approaches, as it exhibits the ability to accurately identify adversarial clients with precise ARL under attack-free scenarios. Experiments are conducted to compare FedRR with benchmark methods on the MNIST dataset, which demonstrates the swift detection capabilities of FedRR when adversarial attacks occur. Future endeavors will encompass scenarios where clients are heterogeneous or samples collected by multiple clients over time are correlated.

\section*{Appendix A: The Proof of Theorem~\ref{thm}}

In the FedAvg procedure, we have the recursive relationship
\begin{equation}
  \label{eq:iteration}
\bmn w_t^{(k)} = \bmn F(\bmn w_t; \mathcal{D}_{t,k}), \quad \bmn w_{t+1} = \frac 1 K \sum_{k=1}^K \bmn w_t^{(k)}   
\end{equation}
for all $t = 1,2,\ldots$ and $k = 1,\ldots, K$. Here $F(\bmn w; \mathcal{D})$ refers to the process that a client calcualte the updated model parameters from the global model parameters $\bmn w$ and the data $\mathcal{D}$. The initial aggregated model parameters is  $\bmn w_1$. The residual
\begin{equation}
  \label{eq:residual}
r_t^{(k)} = \norm{\Delta \bmn w_t^{(k)} - \bmn P_{\bmn B} \Delta \bmn w_t^{(k)}} = \norm{(\bmn I - \bmn P_{\bmn B}) (\bmn F(\bmn w_t; \mathcal{D}_{t,k})-\bmn w_t)} := G(\bmn w_t; \mathcal{D}_{t,k}).  
\end{equation}

\begin{enumerate}[(1)]
\item Since $\mathcal D_{t,1},\ldots, \mathcal D_{t,K}$ are IID, we know that  $r_t^{(1)},\ldots, r_t^{(K)} | \bmn w_t$ are IID. Therefore, $P(\tilde {\bmn r}_t = \bmn p|\bmn w_t) = 1/K!$ for any permutation $\bmn p$, which leads to  
  \[P(\tilde {\bmn r}_t = \bmn p) = \int P(\tilde {\bmn r}_t = \bmn p|\bmn w_t)p(\bmn w_t) d\bmn w_t = \int \frac 1 {K!} p(\bmn w_t) d\bmn w_t = \frac 1 {K!}. \]
\item Denote $\bmn r_t = (r_t^{(1)},\ldots, r_t^{(K)})$ and $\mathcal D_{t} = \{\mathcal D_{t,1},\ldots, \mathcal D_{t,K}\}$, and we define $\bmn r_t = \{G(\bmn w_t; \mathcal D_{t,k})\}_{k=1}^K$.

  We first claim that $\bmn w_{t+1}$ and $\tilde{\bmn r}_t$ are independent, according to Lemma~\ref{lem} below, by letting $N = K$, $\bmn X_k = \mathcal D_{t,k}$, $\bmn Y = \bmn w_{t}$, $\bmn w_{t+1} = \bmn h(\bmn w_t, \mathcal D_{t,1},\ldots, \mathcal D_{t,K})$, and $g(\cdot,\cdot)=G(\cdot,\cdot)$. 

  \begin{lemma}
    Let $\bmn X_1,\ldots, \bmn X_N|\bmn Y$ be IID where $F(\cdot)$ is a continuous distribution on set $\mathcal X$. Let $g(\bmn x)$ be a continuous function $\mathcal X\to\RR$ and $\bmn R$ the rank statistics of $g(\bmn Y, \bmn X_1),\ldots,g(\bmn Y, \bmn X_N)$. Let $\bmn h(\bmn y, \bmn x_1,\ldots, \bmn x_N)$ be symmetric on $\bmn x_1, \ldots, \bmn x_N$, i.e.,
    \[\bmn h(\bmn y, \bmn x_1,\ldots, \bmn x_N) = \bmn h(\bmn y, \bmn x_{\bmn p(1)},\ldots, \bmn x_{\bmn p(N)})\]
    for any permutation $\bmn p$ of $[N]$. Then $\bmn h(\bmn Y, \bmn X_1,\ldots, \bmn X_N)$ and $\bmn R$ are independent.
    \label{lem}
  \end{lemma}

  \begin{proof}
    For any permutation $\bmn p$ and set $A\subset \mathcal X$, we have
    \[P(\bmn h(\bmn Y, \bmn Z_1,\ldots, \bmn Z_N)\in A,\bmn R_{\bmn Z} =\tilde{\bmn r}| \bmn Y) = P(\bmn h(\bmn Y, \bmn X_{\bmn p(1)},\ldots, \bmn X_{\bmn p(N)}) \in A,\bmn R_{\bmn X\circ \bmn p} = \tilde{\bmn r}
      | \bmn Y)\]
    by considering the variable substitution $(\bmn Z_1, \ldots, \bmn Z_N) = (\bmn X_{\bmn p(1)},\ldots, \bmn X_{\bmn p(N)})$. Therefore, for any permutation $\tilde{\bmn r}, \tilde{\bmn r}'$, we have
    \[P(h(\bmn Y, \bmn X_1,\ldots, \bmn X_N)\in A,\bmn R_X = \tilde{\bmn r} | \bmn Y) = P(h(\bmn Y, \bmn X_1,\ldots, \bmn X_N) \in A,\bmn R_Z = \tilde{\bmn r}'      | \bmn Y)). 
    \]
    which indicates
    \[P(h(\bmn Y, \bmn X_1,\ldots, \bmn X_N)\in A,\bmn R_X = \tilde{\bmn r} | \bmn Y)=\frac 1 {N!} P(h(\bmn Y, \bmn X_1,\ldots, \bmn X_N )\in A|\bmn Y). \]
    Then integrate out $\bmn Y$ we have
    \begin{align*}
      P(h(\bmn Y, \bmn X_1,\ldots, \bmn X_N)\in A,\bmn R_X = \tilde{\bmn r}) =& \EE_{\bmn Y} P(h(\bmn Y, \bmn X_1,\ldots, \bmn X_N)\in A,\bmn R_X = \tilde{\bmn r}|\bmn Y)      \\
      =& \EE_{\bmn Y} \frac{1}{N!}P(h(\bmn Y, \bmn X_1,\ldots, \bmn X_N )\in A|\bmn Y) \\
      =& \frac{1}{N!}P(h(\bmn Y, \bmn X_1,\ldots, \bmn X_N )\in A)
    \end{align*}
For any permutation $\tilde{\bmn r}$ of $[N]$, clearly $P(\bmn R = \tilde{\bmn r}) = 1/N!$, which proof the Lemma. 
  \end{proof}

  We then claim that $\tilde{\bmn r}_t$ is independent with $(\tilde{\bmn r}_{t+1},\ldots, \tilde{\bmn r}_{t+s})$ for any $t, s>0$. From Equations \ref{eq:iteration} and \ref{eq:residual} , we know $\tilde{\bmn r}_{t+s'}$ is a function of $\bmn w_{t+s'-1},\mathcal D_{t+s'-1}$, and thereby a function of $\bmn w_{t+s'-2},\mathcal D_{t+s'-2},\mathcal D_{t+s'-1}$, ..., and finally a function of $\bmn w_{t+1}$, $\mathcal D_{t+1}$,..., $\mathcal D_{t+s'}$. So $\tilde{\bmn r}_{t+1},...,\tilde{\bmn r}_{t+s}$ is a function of $\bmn w_{t+1}$, $\mathcal D_{t+1}$,..., $\mathcal D_{t+s}$.
  \begin{align*}
    &p(\tilde{\bmn r}_t, \bmn w_{t+1},  \mathcal D_{t+1},\ldots, \mathcal D_{t+s})     \\
    =& p(\tilde{\bmn r}_t, \bmn w_{t+1}) p(\mathcal D_{t+1},\ldots, \mathcal D_{t+s})     \\
    =& p(\tilde{\bmn r}_t)p(\bmn w_{t+1})p(\mathcal D_{t+1},\ldots, \mathcal D_{t+s})     
  \end{align*}
  So $\tilde{\bmn r}_t$ is independent with $\bmn w_{t+1}, \mathcal D_{t+1},\ldots, \mathcal D_{t+s}$. 
  Therefore the claim is validated.

  Finally, we note that
  \[p(\tilde{\bmn r}_1,\tilde{\bmn r}_2,\ldots, \tilde{\bmn r}_t) = p(\tilde{\bmn r}_1)p(\tilde{\bmn r}_2,...,\tilde{\bmn r}_t) = \cdots = \Pi_{i=1}^t p(\tilde{\bmn r}_i). \]
  So all $\tilde{\bmn r}_i$'s are independent.

\end{enumerate}

\section*{Appendix B: The algorithm of finding the control limit}
 Algorithm~\ref{Control Limit} evaluates the ARL of a control scheme without attack given subscribed control limit $H$. The control limit corresponding to a prescribed ARL can be found using an iterative search algorithm, such as a bisection search. 

 

\spacingset{1.0} 
\RestyleAlgo{ruled}
\SetKwProg{Init}{Initialization:}{}{}

\begin{algorithm}
\label{Control Limit}
\DontPrintSemicolon
\KwIn{
Control limit $H$, total number of run lengths $M$, number of clients $K$
}
\For{$m = 1,\ldots,M$}{
    \Init{}{
        When $t=0$, initialize $S_t^{(k)}=0$ for $k=1,\ldots, K$.\;    
    }
    \For{$t = 1,2,\ldots,$}{
        Sample $\tilde{\bm r}_t$ as independent random permutation of $\{1,\ldots, K\}$. \;
        Calculate $Z_t^{(k)}$ from  $\tilde{\bm r}_t$ and update $S_t^{(k)}$ according to line~\ref{FedRR:genZ} of Algorithm~\ref{FedRR}.
        \If{$\max\{S_t^{(k)}: k = 1,...,K\} > H$}{
            Stop and record the run lengths $R_m=t$.\;
        }
    }
    Collect $M$ run lengths and calculate estimated ARL $=\frac 1 M \sum_{m=1}^M R_m$. 
}

\caption{Finding Control Limit}
\end{algorithm}

\newpage

\bibliographystyle{apalike}
\spacingset{1}

\bibliography{IISE-Trans}

\begin{thebibliography}{}

\bibitem[Alnajar and Barnawi, 2023]{alnajar2023tactile}
Alnajar, O. and Barnawi, A. (2023).
\newblock Tactile internet of federated things: Toward fine-grained design of fl-based architecture to meet tiot demands.
\newblock {\em Computer Networks}, page 109712.

\bibitem[Andreina et~al., 2021]{andreina2021baffle}
Andreina, S., Marson, G.~A., M{\"o}llering, H., and Karame, G. (2021).
\newblock Baffle: Backdoor detection via feedback-based federated learning.
\newblock In {\em 2021 IEEE 41st International Conference on Distributed Computing Systems (ICDCS)}, pages 852--863. IEEE.

\bibitem[Azad et~al., 2021]{azad2021preventive}
Azad, K. M.~S., Hossain, N., Islam, M.~J., Rahman, A., and Kabir, S. (2021).
\newblock Preventive determination and avoidance of ddos attack with sdn over the iot networks.
\newblock In {\em 2021 International Conference on Automation, Control and Mechatronics for Industry 4.0 (ACMI)}, pages 1--6. IEEE.

\bibitem[Azam et~al., 2021]{azam2021recycling}
Azam, S.~S., Hosseinalipour, S., Qiu, Q., and Brinton, C. (2021).
\newblock Recycling model updates in federated learning: Are gradient subspaces low-rank?
\newblock In {\em International Conference on Learning Representations}.

\bibitem[Bhagoji et~al., 2019]{bhagoji2019analyzing}
Bhagoji, A.~N., Chakraborty, S., Mittal, P., and Calo, S. (2019).
\newblock Analyzing federated learning through an adversarial lens.
\newblock In {\em International Conference on Machine Learning}, pages 634--643. PMLR.

\bibitem[Blanchard et~al., 2017]{blanchard2017machine}
Blanchard, P., El~Mhamdi, E.~M., Guerraoui, R., and Stainer, J. (2017).
\newblock Machine learning with adversaries: Byzantine tolerant gradient descent.
\newblock {\em Advances in neural information processing systems}, 30.

\bibitem[Bouacida and Mohapatra, 2021]{bouacida2021vulnerabilities}
Bouacida, N. and Mohapatra, P. (2021).
\newblock Vulnerabilities in federated learning.
\newblock {\em IEEE Access}, 9:63229--63249.

\bibitem[Bron and Kerbosch, 1973]{bron1973algorithm}
Bron, C. and Kerbosch, J. (1973).
\newblock Algorithm 457: finding all cliques of an undirected graph.
\newblock {\em Communications of the ACM}, 16(9):575--577.

\bibitem[Cao et~al., 2019]{cao2019understanding}
Cao, D., Chang, S., Lin, Z., Liu, G., and Sun, D. (2019).
\newblock Understanding distributed poisoning attack in federated learning.
\newblock In {\em 2019 IEEE 25th International Conference on Parallel and Distributed Systems (ICPADS)}, pages 233--239. IEEE.

\bibitem[Chen et~al., 2017]{chen2017distributed}
Chen, Y., Su, L., and Xu, J. (2017).
\newblock Distributed statistical machine learning in adversarial settings: Byzantine gradient descent.
\newblock {\em Proceedings of the ACM on Measurement and Analysis of Computing Systems}, 1(2):1--25.

\bibitem[Fang et~al., 2020]{fang2020local}
Fang, M., Cao, X., Jia, J., and Gong, N.~Z. (2020).
\newblock Local model poisoning attacks to byzantine-robust federated learning.
\newblock In {\em Proceedings of the 29th USENIX Conference on Security Symposium}, pages 1623--1640.

\bibitem[Guerraoui et~al., 2018]{guerraoui2018hidden}
Guerraoui, R., Rouault, S., et~al. (2018).
\newblock The hidden vulnerability of distributed learning in byzantium.
\newblock In {\em International Conference on Machine Learning}, pages 3521--3530. PMLR.

\bibitem[Hegiste et~al., 2022]{hegiste2022application}
Hegiste, V., Legler, T., and Ruskowski, M. (2022).
\newblock Application of federated learning in manufacturing.
\newblock {\em arXiv preprint arXiv:2208.04664}.

\bibitem[Kairouz et~al., 2021]{kairouz2021advances}
Kairouz, P., McMahan, H.~B., Avent, B., Bellet, A., Bennis, M., Bhagoji, A.~N., Bonawitz, K., Charles, Z., Cormode, G., Cummings, R., et~al. (2021).
\newblock Advances and open problems in federated learning.
\newblock {\em Foundations and Trends{\textregistered} in Machine Learning}, 14(1--2):1--210.

\bibitem[Kone{\v{c}}n{\`y} et~al., 2016]{konevcny2016federated}
Kone{\v{c}}n{\`y}, J., McMahan, H.~B., Yu, F.~X., Richt{\'a}rik, P., Suresh, A.~T., and Bacon, D. (2016).
\newblock Federated learning: Strategies for improving communication efficiency.
\newblock {\em arXiv preprint arXiv:1610.05492}.

\bibitem[Kontar et~al., 2021]{kontar2021internet}
Kontar, R., Shi, N., Yue, X., Chung, S., Byon, E., Chowdhury, M., Jin, J., Kontar, W., Masoud, N., Nouiehed, M., et~al. (2021).
\newblock The internet of federated things (ioft).
\newblock {\em IEEE Access}, 9:156071--156113.

\bibitem[LeCun et~al., 1998]{lecun1998gradient}
LeCun, Y., Bottou, L., Bengio, Y., and Haffner, P. (1998).
\newblock Gradient-based learning applied to document recognition.
\newblock {\em Proceedings of the IEEE}, 86(11):2278--2324.

\bibitem[Liu et~al., 2015]{liu2015adaptive}
Liu, K., Mei, Y., and Shi, J. (2015).
\newblock An adaptive sampling strategy for online high-dimensional process monitoring.
\newblock {\em Technometrics}, 57(3):305--319.

\bibitem[McMahan et~al., 2017]{mcmahan2017communication}
McMahan, B., Moore, E., Ramage, D., Hampson, S., and y~Arcas, B.~A. (2017).
\newblock Communication-efficient learning of deep networks from decentralized data.
\newblock In {\em Artificial intelligence and statistics}, pages 1273--1282. PMLR.

\bibitem[Mei, 2010]{mei2010efficient}
Mei, Y. (2010).
\newblock Efficient scalable schemes for monitoring a large number of data streams.
\newblock {\em Biometrika}, 97(2):419--433.

\bibitem[Montgomery, 2020]{montgomery2020introduction}
Montgomery, D.~C. (2020).
\newblock {\em Introduction to statistical quality control}.
\newblock John Wiley \& Sons.

\bibitem[Mothukuri et~al., 2021]{mothukuri2021survey}
Mothukuri, V., Parizi, R.~M., Pouriyeh, S., Huang, Y., Dehghantanha, A., and Srivastava, G. (2021).
\newblock A survey on security and privacy of federated learning.
\newblock {\em Future Generation Computer Systems}, 115:619--640.

\bibitem[Mu{\~n}oz-Gonz{\'a}lez et~al., 2019]{munoz2019byzantine}
Mu{\~n}oz-Gonz{\'a}lez, L., Co, K.~T., and Lupu, E.~C. (2019).
\newblock Byzantine-robust federated machine learning through adaptive model averaging.
\newblock {\em arXiv preprint arXiv:1909.05125}.

\bibitem[Park et~al., 2021]{park2021sageflow}
Park, J., Han, D.-J., Choi, M., and Moon, J. (2021).
\newblock Sageflow: Robust federated learning against both stragglers and adversaries.
\newblock {\em Advances in neural information processing systems}, 34:840--851.

\bibitem[Paynabar et~al., 2016]{paynabar2016change}
Paynabar, K., Zou, C., and Qiu, P. (2016).
\newblock A change-point approach for phase-i analysis in multivariate profile monitoring and diagnosis.
\newblock {\em Technometrics}, 58(2):191--204.

\bibitem[Rahman et~al., 2023]{rahman2023icn}
Rahman, A., Hasan, K., Kundu, D., Islam, M.~J., Debnath, T., Band, S.~S., and Kumar, N. (2023).
\newblock On the icn-iot with federated learning integration of communication: Concepts, security-privacy issues, applications, and future perspectives.
\newblock {\em Future Generation Computer Systems}, 138:61--88.

\bibitem[Rodr{\'\i}guez-Barroso et~al., 2023]{rodriguez2023survey}
Rodr{\'\i}guez-Barroso, N., Jim{\'e}nez-L{\'o}pez, D., Luz{\'o}n, M.~V., Herrera, F., and Mart{\'\i}nez-C{\'a}mara, E. (2023).
\newblock Survey on federated learning threats: Concepts, taxonomy on attacks and defences, experimental study and challenges.
\newblock {\em Information Fusion}, 90:148--173.

\bibitem[Rodr{\'\i}guez-Barroso et~al., 2022]{rodriguez2022dynamic}
Rodr{\'\i}guez-Barroso, N., Mart{\'\i}nez-C{\'a}mara, E., Luz{\'o}n, M.~V., and Herrera, F. (2022).
\newblock Dynamic defense against byzantine poisoning attacks in federated learning.
\newblock {\em Future Generation Computer Systems}, 133:1--9.

\bibitem[Sharma and Kaushik, 2019]{sharma2019survey}
Sharma, S. and Kaushik, B. (2019).
\newblock A survey on internet of vehicles: Applications, security issues \& solutions.
\newblock {\em Vehicular Communications}, 20:100182.

\bibitem[Sun et~al., 2019]{sun2019can}
Sun, Z., Kairouz, P., Suresh, A.~T., and McMahan, H.~B. (2019).
\newblock Can you really backdoor federated learning?
\newblock {\em arXiv preprint arXiv:1911.07963}.

\bibitem[Wang et~al., 2018]{wang2018spatial}
Wang, A., Xian, X., Tsung, F., and Liu, K. (2018).
\newblock A spatial-adaptive sampling procedure for online monitoring of big data streams.
\newblock {\em Journal of Quality Technology}, 50(4):329--343.

\bibitem[Wang et~al., 2017]{wang2017statistical}
Wang, Z., Li, Y., and Zhou, X. (2017).
\newblock A statistical control chart for monitoring high-dimensional poisson data streams.
\newblock {\em Quality and Reliability Engineering International}, 33(2):307--321.

\bibitem[Wu et~al., 2020]{wu2020federated}
Wu, Z., Ling, Q., Chen, T., and Giannakis, G.~B. (2020).
\newblock Federated variance-reduced stochastic gradient descent with robustness to byzantine attacks.
\newblock {\em IEEE Transactions on Signal Processing}, 68:4583--4596.

\bibitem[Xian et~al., 2018]{xian2018nonparametric}
Xian, X., Wang, A., and Liu, K. (2018).
\newblock A nonparametric adaptive sampling strategy for online monitoring of big data streams.
\newblock {\em Technometrics}, 60(1):14--25.

\bibitem[Yan et~al., 2018]{yan2018real}
Yan, H., Paynabar, K., and Shi, J. (2018).
\newblock Real-time monitoring of high-dimensional functional data streams via spatio-temporal smooth sparse decomposition.
\newblock {\em Technometrics}, 60(2):181--197.

\bibitem[Yin et~al., 2018]{yin2018byzantine}
Yin, D., Chen, Y., Kannan, R., and Bartlett, P. (2018).
\newblock Byzantine-robust distributed learning: Towards optimal statistical rates.
\newblock In {\em International Conference on Machine Learning}, pages 5650--5659. PMLR.

\bibitem[Zhao et~al., 2020]{zhao2020pdgan}
Zhao, Y., Chen, J., Zhang, J., Wu, D., Teng, J., and Yu, S. (2020).
\newblock Pdgan: A novel poisoning defense method in federated learning using generative adversarial network.
\newblock In {\em Algorithms and Architectures for Parallel Processing: 19th International Conference, ICA3PP 2019, Melbourne, VIC, Australia, December 9--11, 2019, Proceedings, Part I 19}, pages 595--609. Springer.

\end{thebibliography}
\end{document}